\documentclass[conference]{IEEEtran}
\IEEEoverridecommandlockouts
\usepackage{cite}
\usepackage{amsmath,amssymb,amsfonts}
\usepackage{algorithmic}
\usepackage{graphicx}
\usepackage{textcomp}
\usepackage{xcolor}
\def\BibTeX{{\rm B\kern-.05em{\sc i\kern-.025em b}\kern-.08em
    T\kern-.1667em\lower.7ex\hbox{E}\kern-.125emX}}

\input{my_symbol.sty}
\usepackage{hyperref}
\usepackage{bbold}
\usepackage{amsthm}
\usepackage{tabularx}
\usepackage{booktabs}
\newtheorem{theorem}{Theorem}
\newtheorem{lemma}[theorem]{Lemma}
\newtheorem{proposition}[theorem]{Proposition}
\begin{document}

\title{A Spectral Framework for Graph Neural Operators: Convergence Guarantees and Tradeoffs}

\author{\IEEEauthorblockN{Roxanne Holden}
\IEEEauthorblockA{\textit{Applied Mathematics and Statistics} \\
\textit{Johns Hopkins University}\\
Baltimore, USA \\
rholden3@jh.edu}
\and
\IEEEauthorblockN{Luana Ruiz}
\IEEEauthorblockA{\textit{Applied Mathematics and Statistics} \\
\textit{Johns Hopkins University}\\
Baltimore, USA \\
lrubini1@jh.edu}
}

\maketitle

\begin{abstract}
Graphons, as limits of graph sequences, provide an operator-theoretic framework for analyzing the asymptotic behavior of graph neural operators. Spectral convergence of sampled graphs to graphons induces convergence of the corresponding neural operators, enabling transferability analyses of graph neural networks (GNNs). This paper develops a unified spectral framework that brings together convergence results under different assumptions on the underlying graphon, including no regularity, global Lipschitz continuity, and piecewise-Lipschitz continuity. The framework places these results in a common operator setting, enabling direct comparison of their assumptions, convergence rates, and tradeoffs. We further illustrate the empirical tightness of these rates on synthetic and real-world graphs.
\end{abstract}


\section{Introduction}
\label{sec:intro}
Graph neural networks (GNNs) are widely used in drug discovery \cite{fang2025recent, satheeskumar2025enhancing}, social networks \cite{ben2024enhancing, sharma2024survey}, recommendation systems \cite{wu2022graph}, and Natural Language Processing \cite{liu2022graph, kumar2022natural, wu2021deep}.
GNNs operate on graph-structured data via message passing and aggregation \cite{gnns}, but training on large graphs is computationally expensive.
Studying GNN behavior on \emph{families} of large graphs, exploiting low-dimensional structure such as finite rank or bandlimitedness, offers statistical guarantees and mitigates scalability challenges \cite{xu2018rates, wolfe2013nonparametric}.

A framework for analyzing families of large graphs is given by \emph{graphons}: symmetric measurable functions $\bbW:[0,1]^2 \to [0,1]$ representing the probability of an edge between two nodes placed in the interval $[0,1]$ \cite{whatisgraphon, laplacianvizuete}.
Graphons are limit objects of graph sequences under the cut distance, a metric that accounts for arbitrary relabelings, providing an avenue for analyzing the asymptotic behavior of graph neural operators and their transferability properties \cite{transferabilityruiz}.
As graphons define bounded, symmetric operators, their spectra capture the frequencies on which graph neural operators act, and spectral convergence of graph sequences to graphon limits translates to operator convergence \cite{lovaszlimits}.

Existing spectral arguments for operator convergence of GNNs on graph sequences, however, are derived under heterogeneous assumptions and measured in incomparable metrics (e.g., cut distance, Frobenius norms, or $L^2$ operator norms), making direct comparison of reported rates difficult. The literature provides a collection of bounds that are individually informative \cite{lovaszlimits, levie, transferabilityruiz, Avella_Medina_2020} but collectively hard to interpret. This fragmentation limits the practical usefulness of these results, as the structural reasons why some graph families exhibit markedly faster convergence than others are obscured.

In this work, we provide a unifying analytical perspective that places these heterogeneous convergence results into a single spectral norm framework. We show that existing graph neural operator convergence bounds can be systematically understood as the combination of two components: Weyl-type spectral perturbation inequalities \eqref{eqn:weyls}, which control eigenvalue deviations via operator-norm kernel differences, and graphon operator approximation errors whose rates depend explicitly on structural regularity assumptions. This decomposition yields a principled comparison mechanism, clarifying that variation in reported convergence rates arises primarily from differences in the assumed graphon structure---and consequently in the generality of the results, particularly their invariance to node relabeling---rather than from fundamentally incompatible analytical techniques.

Within this unified viewpoint, we derive explicit spectral convergence rates for piecewise-Lipschitz graphons, a practically relevant regime for which prior work provides network centrality approximation guarantees \cite{Avella_Medina_2020} but not directly comparable spectral upper bounds. These rates are directly applicable to graph neural operator stability and GNN transferability analyses in moderately structured settings where global smoothness is unrealistic but local regularity is plausible.

%We condense convergence results established under disparate assumptions into a single spectral norm perspective that enables direct, rate-level comparison across graph families, transforming a scattered collection of guarantees into a coherent comparative theory. Our focus is on \emph{graph neural operator convergence}, understood as eigenvalue stability under graphon limits. Section~\ref{sec:notation} introduces basic definitions. Section~\ref{sec:bounds} presents convergence rates under no assumptions, Lipschitz continuity, and piecewise-Lipschitz continuity. A comparison of the bounds is given in Section~\ref{ssec:comparison}, with numerical illustrations in Section~\ref{ssec:examples}, and concluding remarks in Section~\ref{sec:conclusion}.
Building on this unified perspective, we establish spectral convergence rates under no regularity, global Lipschitz continuity, and piecewise-Lipschitz continuity, placing these regimes on a common analytical footing. This unified treatment enables direct, rate-level comparison across assumptions, revealing a fundamental tradeoff between convergence speed and generality. Stronger structural regularity yields faster spectral convergence and greater operator stability, but at the cost of less or no invariance to node relabelings. The remainder of the paper is organized as follows. Section~\ref{sec:notation} introduces basic definitions. Section~\ref{sec:bounds} presents the convergence results. Section~\ref{ssec:comparison} compares the bounds across regimes, Section~\ref{ssec:examples} provides numerical illustrations, and Section~\ref{sec:conclusion} concludes.

% The spectral argument for operator convergence of GNNs on graph sequences uses two facts: (i) Weyl's inequality bounds eigenvalue perturbations by operator-norm kernel differences \cite{horn1994topics}; and (ii) graphon convergence in cut distance implies convergence in cut and operator norms under relabeling \cite[Proposition 4]{ruiz2021graphon}. Together, these yield spectral convergence rates from cut distance, at a $\mathcal{O}(1/\sqrt[4]{\log n})$ rate \cite{lovaszlimits, levie}.

% Sharper rates follow from structural assumptions on the graphon. 
% Ruiz et al.\ \cite{transferabilityruiz} assume global Lipschitz continuity under fixed labeling, yielding $\mathcal{O}(\sqrt{\log n / n})$ rates, though this fails under arbitrary relabelings \cite[Figure 4]{whatisgraphon}. 
% Avella Medina et al.\ \cite{Avella_Medina_2020} assume piecewise-Lipschitz structure, giving intermediate $\mathcal{O}(\sqrt[4]{\log n/n})$ rates while permitting flexible labeling within pieces. 
% Lipschitz constants can be estimated in practice via sort-and-smooth or related methods \cite{chan2014consistenthistogramestimatorexchangeable, airoldi2013stochastic, klopp2019optimal}.

%
\section{Preliminary Definitions}
\label{sec:notation}
We consider undirected graphs \( G = (V, E, W) \) consisting of a set of nodes \( V \), a set of edges \( E \subseteq V \times V \), and a weight function \(W: E \to \mathbb{R}\), assigning a real-valued weight \( w_{ij} \) to each edge \( (i,j) \in E \). An undirected graph \( G = (V, E, W) \) can be represented by its symmetric adjacency matrix \( A \in \mathbb{R}^{n \times n} \), where \( A_{ij} = w_{ij} \) if \( (i,j) \in E \), and \( A_{ij} = 0 \) otherwise. We denote a sequence of graphs $\{G_n\}$, with $n \in\naturals\setminus \{0\}$. 

\subsection{Graphons} \label{sbs:graphons}

Graphons, symmetric measurable functions $\bbW: [0, 1]^2 \rightarrow [0,1]$, can be seen both as generative models for graphs and as limit objects of convergent graph sequences \cite{lovaszlimits, borgs2008convergent}. Using a graphon as a generative model, we can construct an $n$-node graph $G_n$ from $\bbW$ by sampling points $u_i \in [0,1]$---associated with nodes $i \in [1, 2, ...,n]$---and assigning edge weight $\bbW(u_i,u_j)$ to edge $(i,j)$ to obtain a complete weighted graph; or, connecting edge $(i,j)$ with probability $p_{ij} = \bbW(u_i,u_j)$ to obtain a stochastic unweighted graph. In either case, we say $G_n$ is sampled from $\bbW$. 

A sequence of graphs $\{G_n\}$ is said to converge if, for every fixed finite motif $F$ (e.g., a triangle or $k$-cycle), the proportion of adjacency-preserving mappings (homomorphisms) from $F$ into $G_n$ stabilizes as $n \to \infty$ \cite{borgs2008convergent, lovaszlimits}. This proportion, called the homomorphism density $t(F,G_n)$, measures how frequently $F$ appears in $G_n$. Graphons are the limits of such sequences, in that the densities of homomorphisms $t(F,G_n)$ converge to the graphon homomorphism density. Explicitly, the graphon homomorphism density $t(F,\mathbf{W})$ %is defined via an integral over $[0,1]$, 
represents the probability of sampling $F$ from the graphon $\mathbf{W}$. We say that $\{G_n\}$ converges to $\mathbf{W}$ if and only if $t(F,G_n) \to t(F,\mathbf{W})$ for all $F$; in this case, $\mathbf{W}$ is the limit graphon. 

\subsection{Cut norm and cut distance}

The cut norm of a graphon $\mathbf{W}$ measures the strongest concentration of connections between vertex subsets
\[
\|\mathbf{W}\|_{\Box} = \sup_{S,T \subset [0,1]} \left| \int_{S \times T} \mathbf{W}(u,v)\text{d}u\text{d}v \right|.
\]
The cut distance between two graphons, $\mathbf{U}$, $\mathbf{V}$, measures how different they are up to relabeling. It is defined as the infimum of their cut norm difference over all measure-preserving bijections $\phi : [0,1] \to [0,1]$:
\begin{equation} \label{eqn:cut_distance_defn}
\delta_{\Box}(\mathbf{U},\mathbf{V}) = \inf_{\phi \in \Phi} \|\mathbf{U} - \mathbf{V}^{\phi}\|_{\Box}.
\end{equation}
This accounts for arbitrary vertex labels and ensures the distance is label-invariant \cite{lovaszlimits}.

Convergent graph sequences in the homomorphism density sense also converge in the cut distance. This is formalized by defining graphons $\bbW_{G_n}$ induced by the graphs $G_n$, which allows comparison of graphs to graphons through the above-defined distance.
%Additionally, every undirected graph $G$ admits a graphon, referred to as the induced graphon, $\bbW_G$. 
For an $n$-node graph $G_n$, an induced graphon is created by constructing a regular partition $I_1 \cup ... \cup I_n$ of $[0,1]$, where $I_j = [(j-1)/n, j/n), 1 \leq j \leq n-1$ and $I_n = [(n-1)/n, 1]$, and with $\mathbb{I}$ as the indicator function defining $$\bbW_{G_n} (u,v)=\sum_{j=1}^n\sum_{k=1}^{n}[A]_{jk}\mathbb{I}(u \in I_j)\mathbb{I}(v \in I_k).$$ 
If $\{G_n\}$ converges to $\bbW$ in the homomorphism density sense,
$\delta_{\Box}(\bbW_{G_n}, \bbW) \to 0 \text{ as } n \to \infty$, establishing the cut distance as the standard metric for graphon convergence \cite{lovaszlimits, borgs2008convergent}.

\subsection{The graphon operator and Weyl's inequality}

\begin{figure}[t]
\centering
\includegraphics[width=\linewidth]{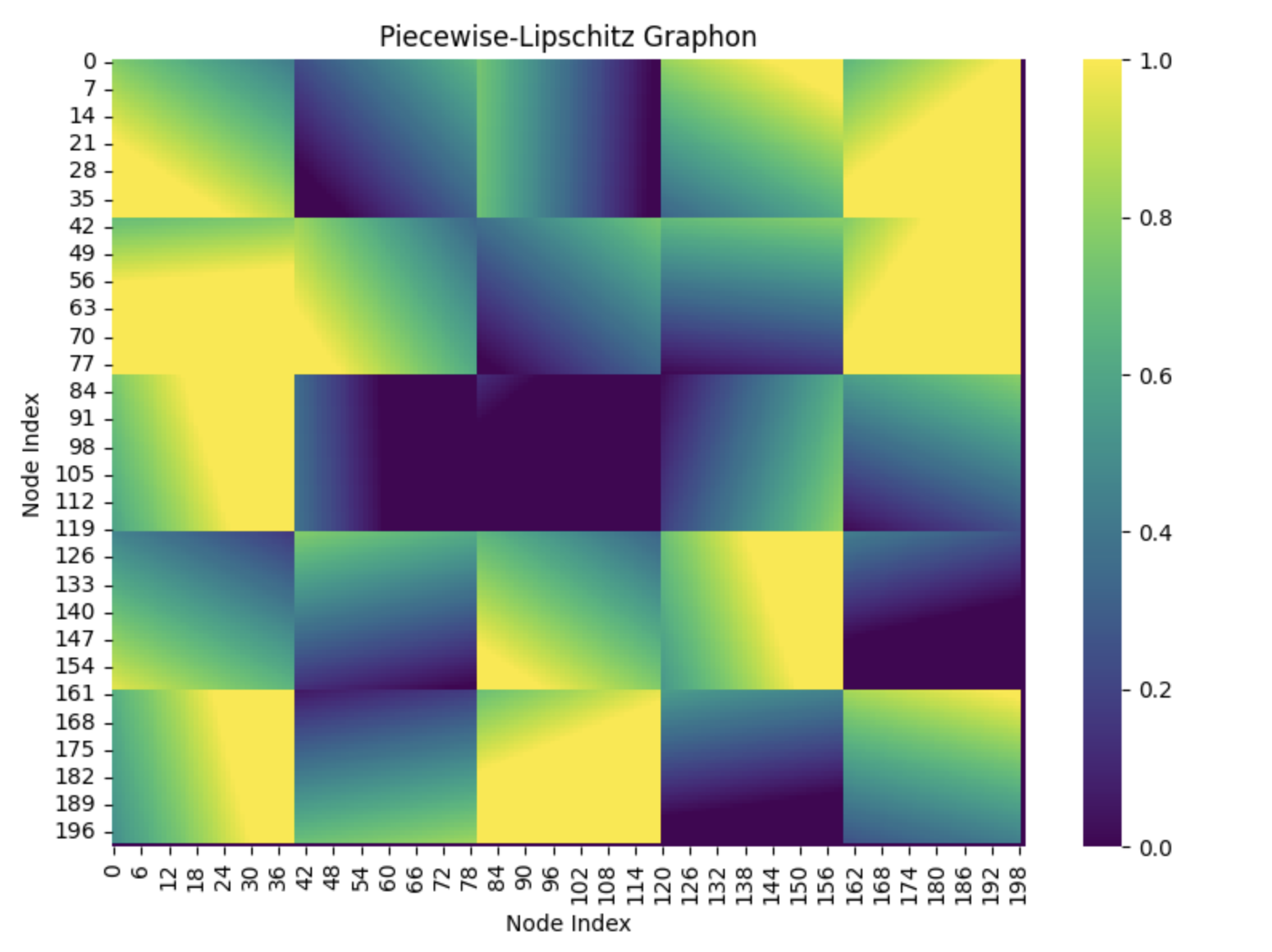}
\caption{Discretized approximation of a piecewise-Lipschitz graphon on a $200 \times 200$ grid, $K=4$, and per-piece Lipschitz constant $\ccalL_k \leq 4$.}
\label{fig:plgraphon}
\end{figure}

As $\bbW$ is bounded and symmetric, it defines an integral linear operator $T_{\bbW}$ on $L^2([0,1])$. More specifically, given a function $X \in L_2([0,1])$---mapping ``graphon nodes'' $u \in [0,1]$ to the reals---the graphon operator acts on it as:
\begin{equation} \label{eqn:graphon_op}
T_\bbW X = \int_{0}^1 \bbW(u, \cdot) X(u) \text{d} u.
\end{equation}
This is the operator that parametrizes graphon convolutions in graphon neural operators such as graphon neural networks (WNNs), which are the idealized limit objects of GNNs on sequences of graphs converging to graphons \cite{transferabilityruiz, ruiz2021graphon, levie}. 

Since graphon operators are compact, self-adjoint operators on $L^2([0,1])$, their spectra consist of real eigenvalues that capture the frequencies on which graphon neural operators act \cite{ruiz2021graphon}. To analyze convergence of graph neural operators to their graphon counterparts, we need control over how the spectra of sampled graph operators deviate from the limiting graphon operator. Weyl's inequality adapted to graphons provides this type of stability guarantee, bounding the difference between eigenvalues of two operators by their operator-norm distance \cite{horn1994topics}.
%Finally, we leverage Weyl's inequality adapted to graphons.
For a graph $G_n$, sampled from a graphon $\bbW$, with induced graphon $\bbW_{G_n}$,  we have 
\begin{equation} \label{eqn:weyls}
|\lambda_i(\bbW)-\lambda_i(\bbW_{G_n})| \leq \|T_{\bbW} - T_{\bbW_{G_n}}\|_{2}
\end{equation}
%Consequently, a small distance in terms of the operator norm between an induced graphon operator of a sampled graph and a graphon guarantees stability of their eigenvalues. 
where in a slight abuse of notation we use $\|\cdot\|_{2}$ to denote %the $L_2$ norm and 
the operator norm induced by the $L_2$ norm. 
\section{Graph Neural Operator Convergence Rates}
\label{sec:bounds}
This section presents operator-level convergence results for graph sequences to their graphon limits. The rates are derived by combining graphon convergence bounds with Weyl's inequality \eqref{eqn:weyls}, transferring these bounds to operator spectra. We focus on \emph{graph neural operator convergence}, measured via eigenvalue stability. Each subsection states structural assumptions on the graphon and the resulting convergence rate.

\subsection{Standard case}
\label{ssec:standard}

Without assumptions on the limit graphon $\bbW$, GNN convergence bounds via Weyl's inequality \eqref{eqn:weyls} follow directly from cut distance convergence, which can be related to the operator norm as summarized in Lemma \ref{prop4ruiz}. For a symmetric kernel $\bbK: [0,1]^2 \to [-1,1]$ (e.g., $\bbW-\bbW'$), this lemma connects its $L_2$-induced operator norm to its cut norm.

\begin{lemma}[Adapted from \cite{ruiz2021graphon}] \label{prop4ruiz}
Let $\bbK: [0,1]^2 \to [-1,1]$.
%, where $\phi$ is the measure preserving bijection that achi \red{??}. 
Then, 
    \begin{align}
    \|\bbK\|_{\Box} \leq \|T_\bbK\|_{2} \leq \sqrt{8\|\bbK\|_{\Box}}.
    \end{align}
\end{lemma}

For sequences of stochastic graphs $\{G_n\}$ sampled from $\bbW$ as described in Sec. \ref{sbs:graphons}, explicit convergence rates for the cut distance \eqref{eqn:cut_distance_defn} are given by \cite[Second Sampling Lemma]{lovaszlimits}, adapted in Lemma \ref{ssl} below.

\begin{lemma}[Second Sampling Lemma] \label{ssl}
Let $G_n$ be a graph randomly sampled from an arbitrary graphon $\bbW$ with associated induced graphon $\bbW_{G_n}$. With probability at least $1 - \exp(-n/(2\log n))$,
    \begin{align}
\delta_\square(\bbW_{G_n},\bbW) = \inf_{\phi \in \Phi}\|{\bbW_{G_n}} - \bbW^{\phi}\|_{\Box} \leq \frac{22}{\sqrt{\log n}}
    \end{align}
where $\Phi$ is the set of measure-preserving bijections on $[0,1]$.
\end{lemma}

Combining Weyl's inequality with Lemmas \ref{prop4ruiz} and \ref{ssl} recovers the first convergence rate for graphon neural operators, without any structural assumptions on $\bbW$.

\begin{proposition}[Standard case]
Let $G_n$ be a graph randomly sampled from a graphon $\bbW$ with induced graphon $\bbW_{G_n}$. With probability at least $1 - \exp(-n/(2\log n))$, the eigenvalues of the corresponding operators satisfy
\[
|\lambda_i(\bbW_{G_n}) - \lambda_i(\bbW)| = \mathcal{O}\!\left(\frac{1}{\sqrt[4]{\log n}}\right).
\]
\end{proposition}
\begin{proof}
%Combining Weyl's inequality and Lemmas \ref{prop4ruiz} and \ref{ssl} yields an upper bound on the rate of the convergence of the spectra of $G_n$---and thus of the associated graph neural operators---without structural assumptions on $\bbW$.
%Let $G_n$ be a graph randomly sampled from graphon $\bbW$ with associated induced graphon $\bbW_{G_n}$ and 
Let $\phi$ be the measure preserving bijection achieving the infimum in Lemma \ref{ssl}. As applying any measure preserving bijection over the unit interval will not affect the spectrum of $T_\bbW$, by Weyl's inequality \eqref{eqn:weyls} and Lemma \ref{prop4ruiz} we have
\begin{equation}
    \begin{split}
        |\lambda_i(\bbW_{G_n}) - \lambda_i(\bbW)| & = |\lambda_i(\bbW_{G_n}) - \lambda_i(\bbW^{\phi})|\\
        &\leq \|T_{\bbW_{G_n}} - T_{\bbW^{\phi}}\|_2 \\
        &\leq \sqrt{8\inf_{\phi \in \Phi}\|\bbW_{G_n} - \bbW^{\phi}\|_{\Box}}\\
        \end{split}
    \end{equation}
Lemma \ref{ssl} implies the stated rate $\mathcal{O}(\tfrac{1}{\sqrt[4]{\log(n)}})$. 

\end{proof}

Under standard assumptions, the resulting spectral bounds are widely applicable due to their permutation-invariance: vertex relabelings have no effect on the estimate. Although such bounds are asymptotically valid, they are often too loose to yield meaningful insight at finite graph sizes. We therefore examine the effect of imposing additional structural assumptions on the graphon $\mathbf{W}$.
%Thus, we have a graph neural operator convergence rate of \( \mathcal{O}(1/\sqrt[4]{\log{n}}) \). 

%\red{and induced by the cut distance metric for graph signals in \cite{levie}, which bounds the $L_2$ induced norm of a graphon operator by its cut norm. In particular,}

\subsection{Lipschitz case}
\label{ssec:lipschitz}

At the cost of generality, we can impose Lipschitz continuity as an additional structural assumption on the graphon $\bbW$. Under this assumption, the Hilbert-Schmidt (HS) norm difference between a graphon and a graph randomly sampled from it can be bounded as follows.

\begin{lemma}[Adapted from \cite{transferabilityruiz}]
\label{lipbound}
Let $G_{n}$ be a graph randomly sampled from a Lipschitz graphon $\bbW$ with associated induced graphon $\bbW_{G_{n}}$. With probability at least $1-\chi \times [1-2x_1] \times [1-x_2]$,
\begin{equation}
\begin{split}
\|\bbW - \bbW_{G_{n}}\|_2 &\leq \frac{2\ccalL_{\bbW}}{n}\log\!\left(\frac{(n+1)^2}{\log(1-x_1)^{-1}}\right)\\
&\quad+ \frac{1}{n}\sqrt{4n\log \!\Big(\tfrac{2n}{\chi}\Big)},
\end{split}
\end{equation}
where $\ccalL_{\bbW}$ is the Lipschitz constant of the graphon, $x_1, x_2 \in (0, 0.3]$, $\chi > 0$, and $n \geq \tfrac{4}{x_2}$.
\end{lemma}

Combining this bound with Weyl's inequality \eqref{eqn:weyls} yields the following operator convergence rate.

\begin{proposition}[Lipschitz case]
Let $G_n$ be a graph randomly sampled from a Lipschitz graphon $\bbW$ with induced graphon $\bbW_{G_n}$. With high probability,
\[
|\lambda_i(\bbW_{G_{n}}) - \lambda_i(\bbW)| = \mathcal{O}\!\left(\sqrt{\tfrac{\log n}{n}}\right).
\]
\end{proposition}

\begin{proof}
Weyl's inequality gives
$|\lambda_i(\bbW_{G_{n}}) - \lambda_i(\bbW)| \leq \|T_{\bbW_{G_n}} - T_{\bbW}\|_2$.
Since $\|T_\bbK\|_2 \leq \|\bbK\|_2$, $\|T_{\bbW_{G_n}} - T_{\bbW}\|_2 \leq \|\bbW_{G_{n}} - \bbW\|_2$, and Lemma~\ref{lipbound} directly implies the stated rate $\mathcal{O}(\sqrt{\log(n)/n})$.
\end{proof}

Imposing global Lipschitz regularity yields substantially faster convergence rates than those obtained for unrestricted graphons. This improvement, however, comes at the cost of permutation invariance; a suitable labeling of inputs must either be provided or learned. We therefore conclude our analysis by examining a weaker structural assumption.

\subsection{Piecewise-Lipschitz case}
\label{ssec:piecelipschitz}

A more forgiving assumption is piecewise-Lipschitz continuity. 
As in \cite{Avella_Medina_2020}, we consider piecewise-Lipschitz graphons $\bbW$ that can be partitioned into at most $K^2$ ``pieces,'' where $K$ is the total number of non-overlapping intervals in the partition. 
Each piece is Lipschitz with constant $\ccalL_k$, and we let $\ccalL = \max_k \ccalL_k$ for consistency with the global Lipschitz notation in Sec.~\ref{ssec:lipschitz}. 
An example is presented in Figure~\ref{fig:plgraphon}.

Note that the different pieces of the graphon can have different Lipschitz constants, and $\bbW$ can be discontinuous along partition boundaries. Under the piecewise-Lipschitz assumption, the HS norm difference between a graphon and a graph randomly sampled from it can be bounded as follows.

\begin{lemma}[Adapted from \cite{ruiz2021graphon, Avella_Medina_2020}]
\label{plipbound}
Let $G_{n}$ be a graph randomly sampled from a piecewise-Lipschitz graphon $\bbW$ with associated induced graphon $\bbW_{G_{n}}$. With probability at least $(1-\chi) (1-\delta')$,
\begin{equation}
\begin{split}
\|\bbW - \bbW_{G_n}\|_{2} &\leq 2\sqrt{(\ccalL^2-K^2){d_n}^2 + Kd_n}\\
&\quad+\frac{1}{n}\sqrt{4n\log \!\big(\tfrac{2n}{\chi}\big)},
\end{split}
\end{equation}
where $\delta' \in (ne^{-n/5}, e^{-1})$, $\chi > 0$, and $d_n = \tfrac{1}{n} + \sqrt{\tfrac{8\log(n/\delta)}{n+1}}$.
\end{lemma}

Combining this bound with Weyl's inequality \eqref{eqn:weyls} yields the following operator convergence rate.

\begin{proposition}[Piecewise-Lipschitz case]
Let $G_n$ be a graph randomly sampled from a piecewise-Lipschitz graphon $\bbW$ with induced graphon $\bbW_{G_n}$. With high probability,
\[
|\lambda_i(\bbW_{G_n}) - \lambda_i(\bbW)| = \mathcal{O}\!\left(\sqrt[4]{\tfrac{\log n}{n}}\right).
\]
\end{proposition}
\begin{proof}
Weyl's inequality gives $|\lambda_i(\bbW_{G_n}) - \lambda_i(\bbW)| \leq \|T_{\bbW_{G_n}} - T_{\bbW}\|_2$. Since $\|T_\bbK\|_2 \leq \|\bbK\|_2$, Lemma~\ref{plipbound} implies the stated rate $\mathcal{O}(\sqrt[4]{\tfrac{\log n}{n}})$.
\end{proof}

Replacing global Lipschitz regularity with a piecewise Lipschitz assumption preserves permutation invariance within partitions but not across them. The resulting bound occupies an intermediate position between the previous regimes in both permutation robustness and convergence rate: it is more flexible than the globally Lipschitz setting but yields slower rates, while remaining less flexible than the fully permutation-invariant case yet achieving faster convergence.

\subsection{Discussion}
\label{ssec:comparison}
The bounds presented in Section \ref{sec:bounds} illustrate the impact of imposing additional structural assumptions on the graphon $\bbW$, particularly as $n$ grows. Imposing structural assumptions results in substantially faster bounds than those obtained for arbitrary graphons. In particular, the bounds shrink faster under the piecewise-Lipschitz assumption, and even more rapidly under the global Lipschitz assumption.    

GNN transferability bounds typically decompose into two factors: a constant term that depends on structural and architectural assumptions (such as GNN depth), and a rate term that decays asymptotically with $n$, determined by spectral convergence. The focus of this work is on the operator-level rates, which can be directly substituted into any existing transferability bound, e.g., \cite{transferabilityruiz}, to yield guarantees under different structural assumptions on the underlying graphon. 

When considering the application of the operator to a specific input signal (i.e., the function $X$ in \eqref{eqn:graphon_op}), a further consideration in transferability is input signal convergence. In the fixed labeling case, this has been established by showing that graph signals converge to Lipschitz graphon signals in $L_2$~\cite{transferabilityruiz,ruiz2021graphon}. In the general case, Levie et al.~\cite{levie} introduced a cut metric for graphon–signal pairs, which controls convergence of both structure and signals simultaneously. Importantly, incorporating these signal convergence errors does not alter the rates derived here, with the resulting transferability bounds inheriting exactly the same asymptotic decay rates.

\section{Numerical examples}
\label{ssec:examples}

%\red{Here is where we should talk about sort and smooth.}

We illustrate bound performance using three examples: a synthetic graphon, the Cora dataset, and the PubMed dataset. The real-world dataset statistics are provided in Table \ref{tab:datasets}.

\begin{table}
\centering
\label{tab:datasets}
\begin{tabular}{lccc}
\toprule
Dataset & Nodes & Edges & Average Degree \\
\midrule
Cora   & 2,485  & 5,069   & 4.079 \\
PubMed & 11,652 & 22,523  & 3.866 \\
\bottomrule
\end{tabular}
\vspace{6pt}
\caption{Dataset statistics.}
\end{table}

The synthetic graphon is generated from a smooth function over $[0,1]^2$ to define connection probabilities. For Cora, we use the full adjacency matrix as the graphon, while for PubMed we use the adjacency matrix of a randomly sampled subset. We normalize the graphon adjacency matrices, via the $L_1$ norm, and reorder by node degree. Leveraging polynomial interpolation we determine an approximate Lipschitz constant for the graphons and particular graphon ``pieces". Sorting and interpolating is consistent with the sort-and-smooth approach used in graphon estimation literature \cite{chan2014consistenthistogramestimatorexchangeable}.

%An order of magnitude estimate plot of the bounds is presented in Figure \ref{fig:boundcompare}. 
Figures~\ref{fig: synth}, \ref{fig: cora}, and \ref{fig: pubmed} compare the first three eigenvalues of sampled subgraphs with the corresponding graphon eigenvalues and their theoretical upper bounds. Figure~\ref{fig: synth} illustrates the behavior on a synthetic graphon, where the separation between the general and structured bounds is most pronounced. 

\begin{figure}[t]
\centering
\includegraphics[width=\linewidth]{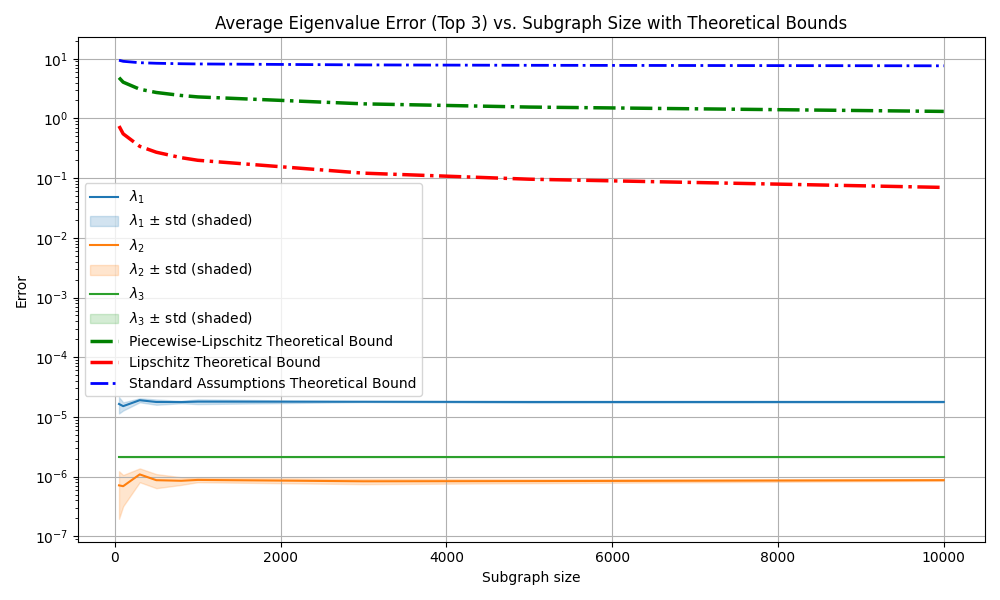}
\caption{Synthetic graphon $\textbf{f} = x \cdot y$ with Lipschitz constant $0.0265$ and largest piecewise-Lipschitz constant $0.353$.}
\label{fig: synth}
\end{figure}

Figures~\ref{fig: cora} and~\ref{fig: pubmed} show similar trends on real-world datasets, though the gaps between bounds are less extreme due to additional empirical structure in the data. 

\begin{figure}[t]
\centering
\includegraphics[width=\linewidth]{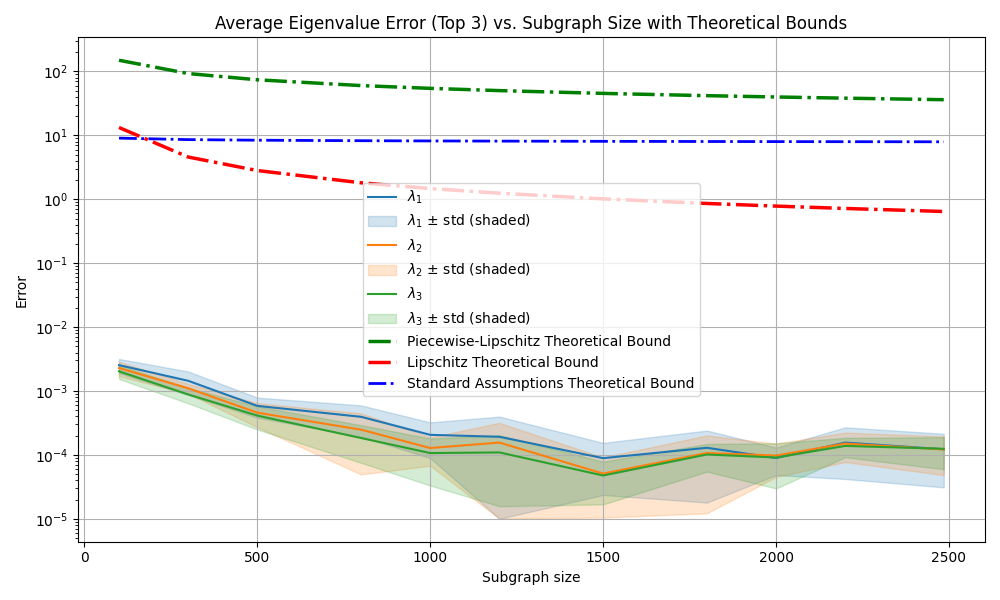}
\caption{Cora graphon with Lipschitz constant $60.653$ and largest piecewise-Lipschitz constant $99.5$.}
\label{fig: cora}
\end{figure}

\begin{figure}[t]
\centering
\includegraphics[width=\linewidth]{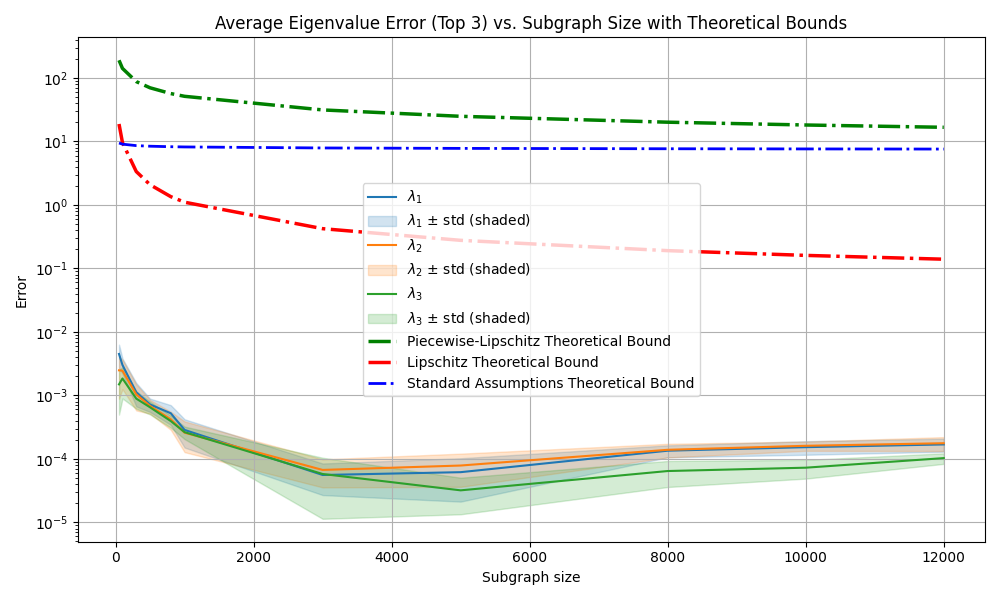}
\caption{PubMed graphon with Lipschitz constant $58.534$ and largest piecewise-Lipschitz constant $199.24$.}
\label{fig: pubmed}
\end{figure}

Across all three settings, bounds derived under stronger structural assumptions perform better overall. While the standard permutation-invariant bound may initially appear tighter than the piecewise-Lipschitz bound, the latter eventually overtakes it as the sample size increases. None of the bounds are fully tight, suggesting that incorporating additional graph structure or refined regularity assumptions could further improve these guarantees.

\section{Conclusion}
\label{sec:conclusion}

Upper bounds on GNN convergence rates are key for analyzing convolutional processing on graph sequences with graphon limits. Few approaches exist, with major contributions by \cite{transferabilityruiz}, \cite{levie}, and \cite{Avella_Medina_2020} presented in the last decade. In this work, we present and compare different graphon-based bounds both theoretically and through numerical examples, consolidating results previously derived under heterogeneous assumptions into a single spectral norm framework. This unifying perspective clarifies the structural trade-offs underlying reported rates and enables direct comparison across graph families that were previously analyzed in isolation.

In addition, we derive explicit spectral upper bounds for piecewise-Lipschitz graphons, providing rate characterizations that are directly applicable to GNN stability and transferability analyses in moderately structured regimes. While none of the existing bounds are tight, our comparative framework highlights where additional structural information can most effectively sharpen guarantees. These results extend naturally to large-scale graphs and offer new opportunities to improve both theoretical understanding and practical deployment of GNNs in real-world tasks.

% -------------------------------------------------------------------------
\bibliographystyle{IEEEbib}
\bibliography{refs}

\end{document}